\documentclass{siamltex}
\usepackage[dvips]{graphicx}
\usepackage{amsfonts}
\usepackage{amssymb}
\usepackage{color}
\usepackage{multirow}
\usepackage{booktabs}
\usepackage[titletoc,toc,title]{appendix}
\usepackage{amsmath,latexsym}
\usepackage{caption}
\usepackage{subfigure}
\usepackage{verbatim}
\newtheorem{remark}{Remark}[section]

\title{An efficient iterative thresholding method for image segmentation
\thanks{We thank Prof. Tony Chan, Zuowei Shen, Xuecheng Tai and Xiaoqun Zhang for helpful discussions and suggestions.  This research was supported in part by the Hong Kong Research Grants Council
  (GRF grants  605513 and 16302715, CRF grant C6004-14G, and
  NSFC-RGC joint research grant N-HKUST620/15).}}
{
\author{
Dong Wang
\thanks{Department of Mathematics, Hong Kong University of Science and Technology,
Clear Water Bay, Kowloon, Hong Kong, China.  (dwangaf@connect.ust.hk).
 }
\and Haohan Li
\thanks{Department of Mathematics,  Hong Kong University of Science and Technology,
Clear Water Bay, Kowloon, Hong Kong, China.  (hlibb@connect.ust.hk).
}
\and Xiaoyu Wei
\thanks{Department of Mathematics,  Hong Kong University of Science and Technology,
Clear Water Bay, Kowloon, Hong Kong, China.  (xweiaf@connect.ust.hk).
}
\and Xiaoping Wang
\thanks{Corresponding author. Department of Mathematics,  Hong Kong University of Science and Technology,
Clear Water Bay, Kowloon, Hong Kong, China.  (mawang@ust.hk).
 }
}

\begin{document}

\maketitle

\renewcommand{\thefootnote}{\arabic{footnote}}
\begin{abstract}
We proposed an efficient iterative thresholding method for multi-phase image segmentation. The algorithm is based on minimizing  piecewise constant Mumford-Shah functional in which the contour length (or perimeter) is approximated by a non-local multi-phase energy.  The minimization problem is solved by an iterative method. Each iteration consists of  computing simple convolutions followed by a thresholding step. The algorithm is easy to implement and has the optimal complexity $O(N \log N)$ per iteration. We also show that the iterative algorithm has the total energy decaying property.  We present some numerical results to show the efficiency of our method.
\end{abstract}



\pagestyle{myheadings}
\thispagestyle{plain}
\markboth{ D. Wang, H. Li, X. Wei, X. Wang}{An efficient iterative thresholding method for image segmentation}

\section{Introduction}
Image segmentation is one of the fundamental tasks in image processing.
In broad terms, image segmentation is the process of partitioning a digital image
into many segments according to a characterization of the image.
The motivation behind this is to determine which part of an image is meaningful for analysis.  It is one of the fundamental problems in computer
vision. Many practical applications require image segmentation, like content-based
image retrieval, machine vision, medical imaging, object detection and traffic control systems \cite{mitiche2010variational}.

The variational method enjoyed tremendous  success in  image segmentation.
In this method, a particular energy is chosen and  minimized  to give
a segmentation of an image.  The Mumford-Shah model \cite{mumford1989optimal} is  the  most successful model and
 has been studied extensively in the last 20 years.
%The model proposed an energy minimization problem which approximates a given
%image by a optimal piece-wise smooth approximation.
More precisely,
the Mumford-Shah model was formulated as follows:
\begin{align}
E_{MS}(u,\Gamma) = \int_{D\setminus\Gamma}|\nabla u|^2dx + \mu Length(\Gamma) + \lambda\int_D{(u-f)}^2dx \label{MS}
\end{align}
Here, $\mu$ and $\lambda$ are positive parameters.  $\Gamma$ is  a closed subset of $D$ given by the union of a finite number of curves. It represents the set of edges (i.e. boundaries of homogeneous regions) in the image $f$. The function $u$ is the piecewise smooth approximation to $f$.
%Because of the Dirichlet integral taken over the set $D\setminus \Gamma$ the energy, $u $ is forced to be smooth in each connected component of $D\setminus \Gamma$. However, it is allowed to have jumps across the curves that make up $\Gamma$.
%
%$f: D \rightarrow R$ is the image, and $u: D\rightarrow R$
%is differentiable in $D\setminus \Gamma$ but may be discontinuous across $\Gamma$.
%The first term is a smooth term which affects the partition boundaries to be smooth (with out small fragments).
%The second term measures the length of the boundary of $\Gamma$.
%The last term measures how well $u$ approximates the image $f$.
%Typically, $u$ is the piecewise smooth approximation to $f$.
Due to the non-convexity of~(\ref{MS}), the minimization problem is difficult to solve numerically \cite{ambrosio1990approximation}.

A  useful simplification of (\ref{MS}) is to restrict the minimization to functions (i.e. segmentations) that take a finite number of values. The resulting model is commonly referred to as the piecewise constant Mumford-Shah model.  In particular,  we have the following two-phase Chan-Vese model~\cite{chan2001active, vese2002multiphase}:
\begin{align}
E_{CV}(\Sigma, C_1, C_2) =\lambda  Per(\Sigma;D)+\int_{\Sigma}{(C_1-f)}^2dx+\int_{D\setminus\Sigma}{(C_2-f)}^2dx\label{CV}
\end{align}
where $\Sigma$ is the interior of a closed curve and $Per(.)$ denotes the perimeter.  $C_1$ and $C_2$ are averages of $f$ within $\Sigma$ and $D \setminus \Sigma$ respectively.
%The first term is a regularization term, which can measure the length of the closed curve $\partial\Sigma$ and the area of the interior region $\Sigma$. Generally, $\partial\Sigma$ is called an active curve. $C_1$ and $C_2$ are called region parameters.
%In ~\cite{chan2001active}, the level set method ~\cite{osher2006level} was used to solve (\ref{CV}).
%A threshold dynamics method was developed in ~\cite{esedog2006threshold} based on MBO type scheme.
%The minimization of (\ref{CV}) is achieved by alternating of the solution of a linear parabolic PDE and simple thresholding.
%The drawback of this active active contour is the existence of local minima.
%To overcome the problem, Bresson and so on~\cite{bresson2007fast}, capitalize a
%fast global minimization method based on convex relaxzation.
%Recently, a continuous
%max-flow approach was introduced in ~\cite{wei2015primal}, in which the dual
%max-flow problem was solved accurately by a primal dual algorithm.
%%
%There have been many works done on devising efficient and robust algorithms for
The level set method was used to solve the minimization problem for  the piecewise constant Mumford-Shah functional (\ref{CV}).  Let $\phi(x): D\rightarrow R$ be  a Lipschitz continuous function  with $\Sigma=\{x\in D: \phi(x)>0\}$ and $D\setminus\Sigma=\{x\in D: \phi(x)<0\}$. We can rewrite (\ref{CV})  as
\begin{align}
E_{CV}(\phi, C_1, C_2) =  \int_D\{\lambda |\nabla H(\phi)|+H(\phi){(C_1-f)}^2+(1-H(\phi)){(C_2-f)}^2\}dx \label{CV2}
\end{align}
where $H(\cdot): R\rightarrow R$ is the Heaviside function
\[
H(\xi)=
\begin{cases}
0 &\text{if } \xi<0,\\
1 &\text{if } \xi \geq0.
\end{cases}
\]
In practice, a regularized version of $H$ denoted by $H_\varepsilon$ is used.
Then the Euler-Lagrange equation of (\ref{CV2}) with respect to $\phi$ is given by
\begin{align}
\frac{\partial\phi}{\partial t} =
-H'_\varepsilon(\phi)\{-\{{(C_1-f)}^2-{(C_2-f)}^2\}
 +\lambda \nabla\cdot(\frac{\nabla\phi}{|\nabla\phi|})\}\label{LCV}
\end{align}
where
\begin{align*}
C_1 = \frac{ \int_D{H(\phi)}fdx }{ \int_D{H(\phi)dx} }
\quad \text{and} \quad
C_2 = \frac{ \int_D{(1-H(\phi))}fdx} { \int_D{(1-H(\phi))dx} }
\end{align*}
Equation (\ref{LCV}) is nonlinear and requires regularization when  $|\nabla \phi|=0$.  Various modifications are used in order to solve the equation more efficiently \cite{ambrosio1990approximation,braides1998approximation,tsai2001curve,vese2002multiphase}.

Esedoglu et al. ~\cite{esedog2006threshold}  proposed a phase-field approximation of (\ref{CV}) in which the Ginzburg-Landau functional is used to approximate the perimeter:
\begin{align}
&E_{MS}^\varepsilon(u,C_1, C_2)  \nonumber\\
=& \int_{D}\left\{ \lambda\left( \varepsilon|\triangledown u|^2 +\frac{1}{\varepsilon} W(u)\right)+u^2(C_1-f)^2+(1-u)^2(C_2-f)^2\right\} dx  \label{ms}
\end{align}
where $\varepsilon>0$ is the approximate interface thickness and $W(\cdot)$ is a double-well potential.
Variation of  (\ref{ms}) with respect to $u$ yields the following gradient descent equation:
\[
u_t=\lambda\left(2\epsilon \Delta u -{1\over\epsilon} W'(u)\right)-2 \{ u (C_1-f)^2+(u-1)(C_2-f)^2\}
\]
which can be solved efficiently by an MBO based threshold dynamic method that works by alternating the solution of a linear (but non-constant coefficient) diffusion equation   with thresholding.
% Each iteration contains two steps:
%\begin{itemize}
%\item Step 1:  Let $v(x)=S(\delta t) u^k(x)$, where $S(\delta t)$ is the propagator (by $\delta t$) %of the linear heat equation
%\begin{equation} \label{ht1}  w_t=\lambda\Delta w -\frac{1}{\sqrt{\pi \delta t}} ( w (c_1 f)^2+(w-1) (c_2-f)^2)
%%\item Step 2:  Set
%\begin{align}
%u^{k+1}(x)=\left\lbrace\begin{array}{cl}
%1 & \text{if} \, v(x) \in (-\infty, {1\over 2}] \\
%0 & \text{if} \, v(x) \in ({1\over 2}, \infty)
%\end{array}\right.
%\end{align}
%\end{itemize}

%In ~\cite{chambolle1999discrete}, Chambolle and Dalmaso have considered a nonlocal approximation of (\ref{MS}) based on adaptive finite elements. Recently, a continuous max-flow approach was introduced in ~\cite{wei2015primal}, in which a dual max-flow problem was solved accurately by a primal-dual algorithm.

In a series of papers \cite{dong2010frame,dong2012wavelet,dong2013x,shen2011accelerated}, a frame-based model was introduced  in which the perimeter  term was approximated  via framelets.  The method was used to  capture  key features of biological structures. The model can also be fast implemented using split Bregman method \cite{goldstein2009split}.

In \cite{cai2013two}, a two-stage segmentation method is proposed. In the first stage, the authors apply the split Bregman method\cite{goldstein2009split} to find the minimizer of a convex variant of the Mumford-Shah functional.  In the second stage, a K-means clustering algorithm is used to choose $k-1$ thresholds automatically to segment the image into $k$ segments. One of the advantages of this method is that there is no need to specify the number of segments before finding the minimizer.  Any $k$-phase segmentation can be obtained by choosing $k-1$ thresholds after  the minimizer is found.

Chan et al. \cite{chan_algorithms_2006} considered a convex reformulation to part of the Chan-Vese model. Given fixed values of $C_1$ and $C_2$, a global minimizer can be found. It is then demonstrated in \cite{yuan_study_2010} that this convex variant can be regarded as a continuous min-cut (primal) problem, and a corresponding continuous max-flow problem can be formulated as its dual. Efficient algorithms are developed by taking advantage of the strong duality between the primal and the dual problem, using the augmented Lagrangian method  or the primal-dual method (see \cite{wei2015primal, yuan_study_2010} and references therein).

The idea of approximating the perimeter of a set by a non-local energy (using heat kernel) \cite{alberti1998non}\cite{miranda2007short} is used by Esedoglu and Otto \cite{esedog2015threshold} to design an efficient  threshold dynamics method for  multi-phase problems with arbitrary surface tensions. The method is also generalized to wetting on rough surfaces in \cite{xu_wang_wang_2016}.
In this paper, we propose an efficient iterative thresholding method for minimizing  the piecewise constant Mumford-Shah functional based on the similar approach.
The perimeter term in (\ref{CV}) is  approximated by a non-local multi-phase energy constructed based on convolution of the heat kernel with the characteristic functions of regions. An iterative algorithm is then derived to minimize the  approximate energy. The procedure works by alternating the convolution step with the thresholding step. The convolution can be implemented efficiently on a uniform mesh using the fast Fourier transform (FFT) with the optimal complexity of $O(N \log N)$ per iteration.  We also show that the algorithm is convergent and has the total energy decaying property.
%Our numerical results show that the proposed method is competitive to  many existing methods for image segmentation, such as Primal-dual method and augmented Lagrangian method.

The rest  of the paper proceeds as follows. In Section 2, we first give the approximate  piecewise constant Mumford-Shah functional.  We then derive the iterative thresholding scheme based on the linearization of the approximate functional. The monotone decrease of the iteration and therefore the convergence of the method is proved (with details given in the appendix).  In Section 3, we present some numerical examples to show the efficiency of the method.

\section{An efficient iterative thresholding method for image segmentation}

In this section, we introduce an iterative thresholding method for image
segmentation based on the Chan-Vese model~\cite{chan2001active}. The perimeter terms in (\ref{CV}) will be approximated by a non-local multi-phase energy constructed based on convolution of the heat kernel with the characteristic functions of regions. The iterative algorithm is then derived as an optimization procedure for the approximate energy. We will also analyse the convergence of the iterative thresholding method.
%%%%%%%%%%%%%%%%%%%%%%%%%%%%%%%%%%%%%%%%%%%%%%%%%%%%%%%%%%%%%%%%%%%%%%%%%
%
%which is motivated by a recent work in~\cite{esedoḡ2015threshold}~\cite{xu2016efficient}.
%Here, we focus on the general form of minimization problem correponding to the image segementation.
%That is, let $\Omega$ denote the domain of an input image.
%Then the task is to find a partition ${\{ \Omega_i\}}_{i=1}^{n}$  of $\Omega$ which minimises the following energy functional
%\begin{align}
%&\min\limits_{\{\Omega_i\}_{i=1}^{n}} \sum\limits_{i=1}^{n}\left[\int_{\Omega_i}g_i(x) d\Omega_i +\lambda \mathcal{R}(\{\Omega_i\}_{i=1}^{n}) \right]\nonumber \\
%& s.t. \cup_{i=1}^n \Omega_i =\Omega  \quad  and \quad \Omega_k \cap \Omega_l =\emptyset \quad for \quad k\neq l,\label{Eq:MinEnergyProb}
%\end{align}
%Where $\mathcal{R}(.)$ is a regularisation term, the data term has the
%form $g_i(x) ={(f(x)-C_i)}^2$, where $f$ is the input image~(i.e. $f(x): \Omega \rightarrow [0,1]$)
%and $C_i$ is average intensity of the region $i$. That is,
%\begin{align}
%C_i=\frac{\int_{\Omega_i} f(x) d\Omega_i}{\int_{\Omega_i} 1 d\Omega_i}. \label{Eq:Ci}
%\end{align}
%
%%%%%%%%%%%%%%%%%%%%%%%%%%%%%%%%%%%%%%%%%%%%%%%%%%%%%%%%%%%%%%%%%%%%%%%%%

\subsection{The approximate Chan-Vese functional} \label{Sec:RepEnergy}
%In general, we consider multi-segment and general image (e.g. colourful image).
 Let $\Omega$ denote the domain of an input image $f$ given by a $d$-dimensional vector. Our task is to find an $n$-phase partition ${\{ \Omega_i\}}_{i=1}^{n}$ of $\Omega$ which minimizes (\ref{CV}) where $ \Omega_i$ represents the region of  the $i^{th}$ phase.
%\begin{align}
%&\min\limits_{\{\Omega_i\}_{i=1}^{n}} \sum\limits_{i=1}^{n}\left[\int_{\Omega_i}g_i(x) d\Omega_i +\lambda \mathcal{R}(\{\Omega_i\}_{i=1}^{n}) \right]\nonumber \\
%& s.t. \cup_{i=1}^n \Omega_i =\Omega  \quad  and \quad \Omega_k \cap \Omega_l =\emptyset \quad for \quad k\neq l,\label{Eq:MinEnergyProb}
%\end{align}
%To represent the energy functional, we firstly need a concrete representation for $\{\Omega_i\}_{i=1}^{n}$.
%We use the characteristic function to represent each region (i.e. $\{\Omega_i\}_{i=1}^{n}$).
% ~\cite{esedog2006threshold} and use the Potts regulariser
%$\mathcal{R}(\{\Omega_i\}_{i=1}^{n})=\sum\limits_{i=1}^{n} | \partial \Omega_i |$ where $| \partial \Omega_i |$ is the perimeter of region $\Omega_i$.
Let $u=(u_1(x),\cdots,u_n(x))$ where $\{u_i(x)\}_{i=1}^{n}$ are the  characteristic functions of the regions $\{\Omega_i\}_{i=1}^{n}$. We then look for $u$ such that
\begin{align}
u=&\mathop{\mathrm{argmin}}\limits_{u\in \mathcal{S}}\sum\limits_{i=1}^{n}\left[\int_{\Omega}u_i(x)g_i(x) d\Omega+\lambda| \partial \Omega_i |\right],  \label{isey}
\end{align}
where $\mathcal{S} =\left\lbrace u=(u_1,\cdots,u_n)\in BV(\Omega): u_i(x)=0,1, \text{and} \sum\limits_{i=1}^{n} u_i=1 \right\rbrace$;
 $|\partial \Omega_i |$ is the length of a boundary curve of the region  $\Omega_i $;
$g_i=
||C_i-f||_2^2$  ($||.||_2$ denotes the $l^2$ vector norm) and
\begin{align}
C_i=\frac{\int_{\Omega} u_ifd\Omega}{\int_{\Omega} u_id\Omega}. \label{Cdef}
\end{align}

%%%%%%%%%%%%%%%%%%%%%%%%%%%%%%%%%%%%%%%%%%%%%%%%%%%%%%%%%%%%%%%%%%%%%%%%%%%%
%\begin{remark}
%Our method is clearly not sensitive to multi-segment or multi-colour (i.e. $d$ can be chosen as any finite number). If one need to involve more information (say another $k$ dimension information besides grayscale value or RGB vector) about the image, it is easy to generalize our method by modifying $f:\Omega\rightarrow [0,1]^{d+k}$. For example, if one would like to reduce the noise of the image, it would be more efficient if involving the grayscale value at some pixels around each pixel to $f$.
%\end{remark}\\
%%%%%%%%%%%%%%%%%%%%%%%%%%%%%%%%%%%%%%%%%%%%%%%%%%%%%%%%%%%%%%%%%%%%%%%%%%%%
It is shown in  \cite{alberti1998non}\cite{miranda2007short}, that when $\delta t \ll 1$, the length of $\partial \Omega_i \cap \partial \Omega_j$ can be approximated by
\begin{align}
|\partial \Omega_i \cap \partial \Omega_j|\approx \sqrt{\frac{\pi}{\delta t}}\int_{\Omega} u_i G_{\delta t} *u_j d\Omega,
\end{align}
where $*$ represents convolution and
\begin{align*}
G_{\delta t}(x)=\frac{1}{4\pi\delta t}exp(-\frac{|x|^2}{4\delta t})
\end{align*}
is the heat kernel.
Therefore,
\begin{align}
|\partial \Omega_i | \approx \sum\limits_{ j=1,j\neq i}^{n} \sqrt{\frac{\pi}{\delta t}} \int_{\Omega} u_i G_{\delta t} *u_jd\Omega.
\end{align}
Hence the total energy can be approximated by
\begin{align}
\mathcal{E}^{\delta t}(u_1,\cdots, u_n) =\sum\limits_{i=1}^{n}\int_{\Omega}\left(u_ig_i+ \lambda\sum\limits_{ j=1,j\neq i}^{n} \frac{\sqrt{\pi}}{\sqrt{\delta t}} u_i G_{\delta t} *u_j\right)d\Omega.   \label{energy1}
\end{align}
Now, (\ref{isey}) becomes
\begin{align}
u=&\mathop{\mathrm{argmin}}\limits_{(u_1,\cdots,u_n)\in \mathcal{S}} \mathcal{E}^{\delta t}(u_1,\cdots, u_n)  \label{Eq:MinApproEnergy}
\end{align}
This is a non-convex minimization problem since $\mathcal{S}$ is not a convex set.  However, we can relax this non-convex problem to a convex  problem  by finding $u=(u_1,\cdots,u_n)$ such that
\begin{align}
u=&\mathop{\mathrm{argmin}}\limits_{(u_1,\cdots,u_n)\in \mathcal{K}} \mathcal{E}^{\delta t}(u_1,\cdots, u_n). \label{Eq:RelaxMinApproEnergy}
\end{align}
where $\mathcal{K}$ is the convex hull of $\mathcal{S}$:
\begin{align}
\mathcal{K}=\left\lbrace u=(u_1,\cdots,u_n)\in BV(\Omega): 0\leq u_i(x)\leq1, \text{and} \sum\limits_{i=1}^{n} u_i=1 \right\rbrace.
\end{align}

\begin{remark}
It is easy to see  that the relaxed minimization problem (\ref{Eq:RelaxMinApproEnergy}) is convex if $C_i (i=1,..n) $ are  constants.
%When $C_i (i=1,..n) $ depend on $u$, we will show  the energy decay property.
\end{remark}

\bigskip

The  following lemma shows that the relaxed problem (\ref{Eq:RelaxMinApproEnergy}) is equivalent to the original problem (\ref{Eq:MinApproEnergy}). Therefore we can solve the relaxed problem (\ref{Eq:RelaxMinApproEnergy}) instead.
\begin{lemma} \label{Lemma}
Let $\mathcal{L}$ be any linear functional defined on $\mathcal{K}$ and $u=(u_1,\cdots, u_n)$. Then
\begin{align}
\mathop{\mathrm{argmin}}\limits_{u\in \mathcal{S}} (\mathcal{E}^{\delta t}(u)+\mathcal{L}(u))=\mathop{\mathrm{argmin}}\limits_{u\in \mathcal{K}} (\mathcal{E}^{\delta t}(u)+\mathcal{L}(u)).
\end{align}
\end{lemma}
\begin{proof} See Appendix A.
\end{proof}

\subsection{Derivation of the iterative thresholding method} \label{Sec:Derivation}
In the following, we show that the minimization  problem (\ref{Eq:MinApproEnergy}) can be solved  by an iterative thresholding method.
Suppose that we have  the $k^{th}$ iteration $(u_1^k, \cdots, u_n^k)\subset \mathcal{S}. $ Let  $g_i^k=||C_i^k-f||_2^2$ with
$$C_i^k=\frac{\int_{\Omega} u_i^kfd\Omega}{\int_{\Omega} u_i^kd\Omega}.$$
Then the energy functional $\mathcal{E}^{\delta t}(u_1,\cdots, u_n)$ with $g_i=g_i^k$ given above can be linearized near the point $(u_1^k,\cdots,u_n^k)$ by
\begin{align}
&\mathcal{E}^{\delta t}(u_1,\cdots, u_n) \approx  \mathcal{E}^{\delta t}(u_1^k,\cdots, u_n^k) \nonumber \\
&+\mathcal{L}(u_1-u_1^k,\cdots,u_n-u_n^k,u_1^k,\cdots,u_n^k)+h.o.t
\end{align}
where
\begin{align} \label{linear1}
\mathcal{L}(u_1,\cdots,u_n,u_1^k,\cdots,u_n^k)=\sum\limits_{i=1}^{n}\int_{\Omega}\left(u_ig_i^k+ \sum\limits_{ j=1,j\neq i}^{n} \frac{2\lambda\sqrt{\pi}}{\sqrt{\delta t}} u_i G_{\delta t} *u_j^k\right)d\Omega \nonumber\\
=\sum\limits_{i=1}^{n}\int_{\Omega}u_i \left(g_i^k+\sum\limits_{ j=1,j\neq i}^{n} \frac{2\lambda\sqrt{\pi}}{\sqrt{\delta t}} G_{\delta t} *u_j^k \right)d\Omega.
\end{align}
We can now  determine the next iteration $(u_1^{k+1}, \cdots, u_n^{k+1})$ by minimizing the linearized functional
\begin{align}
\min\limits_{(u_1,\cdots,u_n)\in \mathcal{K}}\mathcal{L}(u_1,\cdots,u_n,u_1^k,\cdots,u_n^k). \label{Eq:MinLinearEnergy}
\end{align}
%and set the solution to be $(u_1^{k+1},\cdots,u_n^{k+1})$.
%\begin{lemma}
%Denote
%\begin{align}
%\phi_i^k=g_i^k+ \sum\limits_{ j=1,j\neq i}^{n} \frac{2\lambda}{\sqrt{\delta t}} G_{\delta t} *u_j^k.
%\end{align}
%Let
%\begin{align*}
%\Sigma_i^{k+1}=\left\lbrace x: \phi_i^k(x)<\min\limits_{j\neq i}\phi_j^k(x)\right\rbrace
%\end{align*}
%Then, $(u_1^{k+1},\cdots,u_n^{k+1})=(\chi_{\Sigma_1^{k+1}},\cdots,\chi_{\Sigma_n^{k+1}})$ is a solution of (\ref{Eq:MinLinearEnergy}).
%\end{lemma}
Denote
\begin{align}
\phi_i^k:&=g_i^k+ \sum\limits_{ j=1,j\neq i}^{n} \frac{2\lambda\sqrt{\pi}}{\sqrt{\delta t}} G_{\delta t} *u_j^k.\\
&=g_i^k+\frac{2\lambda\sqrt{\pi}}{\sqrt{\delta t}} (1-G_{\delta t} *u_i^k).
\end{align}
We have 
\begin{align} \label{linear3}
\mathcal{L}(u_1,\cdots,u_n,u_1^k,\cdots,u_n^k)=\sum\limits_{i=1}^{n}\int_{\Omega} u_i\phi_i^k d\Omega. 
\end{align}
%By Lemma 3.1, the solution of (\ref{Eq:MinLinearEnergy}) is in $\mathcal{S}$.  Then, the %minimization can be solved by a simple thresholding approach.
 The optimization problem (\ref{Eq:MinLinearEnergy}) becomes minimizing a linear functional over a convex set. It   can be carried out at each $x\in \Omega$ independently.  By  comparing the coefficients $\phi_i^k(x)$ (non-negative) of $u_i(x)$ in the integrand of (\ref{linear3}), it is easy to see that the minimum is attained at
\begin{align}
u_i^{k+1}(x)=\left\lbrace\begin{array}{cc}
1 & \text{if} \, \phi_i^k(x)=\min\limits_{l} \phi_l^k(x),\\
0 & \text{otherwise}.
\end{array}\right.
\end{align}
The following theorem shows  that the total energy $\mathcal{E}^{\delta t}$  decreases in the iteration for any $\delta t>0$.   Therefore, our iteration algorithm always converges to a minimum for any initial partition.
 \begin{theorem}\label{Theorem}
Let $(u_1^{k+1},\cdots,u_n^{k+1})$ be the $k+1^{th}$ iteration derived  above,   we  have
\begin{align}
\mathcal{E}^{\delta t}(u_1^{k+1},\cdots,u_n^{k+1}) \leq \mathcal{E}^{\delta t}(u_1^{k},\cdots,u_n^{k})  \label{Eq:StabilityInProof}
\end{align}
for all $\delta t>0.$
\end{theorem}
\begin{proof}See Appendix B.
\end{proof}

%\newpage
\bigskip

We are then led to the following iterative thresholding algorithm:
\vspace{0.5cm}

{{\bf Algorithm: I}
\it
\begin{itemize}\item[]
\begin{description}
\item[Step 0.] Given an initial partition  $\Omega_1^0, ..., \Omega_n^0\subset\Omega$ and the corresponding $u_1^0=\chi_{\Omega_1^0}, ..., u_n^0=\chi_{\Omega_n^0}$.
%to obtain partition $\Omega_1^{k+1}$, $\cdots$,$\Omega_n^{k+1}$ at time step $t=(k+1)\delta t$ from the partition $\Omega_1^{k}, \cdots,\Omega_n^{k}$ at time $t=k\delta t$.
Set a tolerance parameter $\tau>0.$
\item[Step 1.] Given $k^{th}$ iteration  $(u_1^k, \cdots, u_n^k)\subset \mathcal{S}$, we compute  $g_i^k $ and the following convolutions for $i=1,\cdots,n$:
\begin{align}
\phi_i^k:&=g_i^k+\frac{2\lambda\sqrt{\pi}}{\sqrt{\delta t}} (1-G_{\delta t} *u_i^k)
\end{align}
\item[Step 2.] Thresholding:  Let
\begin{align}
\Omega_i^{k+1}=\left\lbrace x: \phi_i^k(x)<\min\limits_{j\neq i}\phi_j^k(x)\right\rbrace
\end{align}
and define $u_i^{k+1}=\chi_{\Omega_i^{k+1}}$ where $\chi_{\Omega_i^{k+1}}$ represents the charecteristic function of region $\Omega_i^{k+1}$
\item[Step 3.]
Let the normalized $L^2$ difference between successive iterations be
\[   e^{k+1}= \frac{1}{|\Omega|}\int_{\Omega}\sum\limits_{i=1}^n |u_i^{k+1}-u_i^{k}|^2 d\Omega.
\]
If $ e^{k+1}  \leq \tau$, stop. Otherwise, go back to step 1.
\end{description}
\end{itemize}
}
\vspace{0.5cm}

\begin{comment}
\begin{remark}
In  step 1,  the convolution is computed by discretizing the domain. We  fix the computational domain to be $[0,2\pi]^2$, so that $\delta x = 2\pi/N$ for an $N\times N$ image.
\end{remark}
\vspace{0.3cm}

\begin{remark}
For the two phase segmentation,  the algorithm is  slightly simplified since $u_1^{k+1} = \chi_{\{\Phi^k<0\}}$ and $u_2^{k+1} =1-u_1^{k+1}  $  at each step, where
\begin{align*}
\Phi^k=\frac{\sqrt{\delta t}}{2\lambda\sqrt{\pi}}(g_1^k-g_2^k)-G_{\delta t} * (u_1^k - u_2^k)
\end{align*}

Ideally, $\Phi^k$ should be $0$  on the interface between two different phases. That means
	\begin{align}
	\tilde{\lambda}(g_1^k-g_2^k)=G_{\delta t} * ( u_1^k - u_2^k ) \label{Eq:ParameterChoice}
	\end{align}
	It is seen that contour evolution is driven by the competition between the diffusion term and the image data term. Denote the strength of image data term as $\tilde{\lambda}=\frac{\sqrt{\delta t}}{2\lambda\sqrt{\pi}}$.
	
	The left hand side of (\ref{Eq:ParameterChoice}) depends linearly on $\tilde{\lambda}$ while the right hand side is can be thought of as the solution to heat equation at $t=\delta t$
	when taking $ u_1^k - u_2^k$ as initial condition (thus independent of $\tilde{\lambda}$). In principle, $\tilde{\lambda}$ and $\delta t$ should be chosen properly to strike a balance between the two terms.
	
	As will be demonstrated in Section~\ref{Sec:results}, segmentation results are not very sensitive to $\tilde{\lambda}$ and usually it can be set to around $1.0$.

\end{remark}
\end{comment}

\begin{remark}
The  convolutions in Step 1 are computed efficiently using FFT  with a  computational complexity of $O(Nlog(N))$,  where $N$ is the total number of pixels.  Therefore the total computational cost at each iteration is also $O(Nlog(N))$.
\end{remark}

\begin{remark}
In Step 3,  $e^k$ measures   the percentage  of pixels on which $u_i^{k+1}\ne u_i^k$. Therefore the tolerance $\tau$ specifies the threshold of the percentage  of  pixels changing during the iteration  below which the iteration stops.
\end{remark}

%\begin{remark}
%Although we have shown the stability result, we can not choose any large parameter $\delta t$ in the algorithm. It is because the choice of $\delta t$ also affects the accuracy of the approximation of $\mathcal{E}^{\delta t} $ to the energy (\ref{CV}). %We will discuss the choice of $\delta t$ via numerical experiments in Section \ref{Sec:results}.
%\end{remark}

\section{Numerical  Results}
\label{Sec:results}  We now present numerical examples to illustrate the performance of our algorithm.    We implement the algorithm in MATLAB. All the computations are carried out on a MacBook Pro laptop with a 3.0GHz Intel(R) Core(TM) i7 processor and 8GB of RAM.

%We have implemented the aforementioned algorithms in MATLAB.
%We now illustrate the performance
%of our algorithm via several numerical examples.

%The criteria to stop the iteration is when there is no change in characteristic function $u$ within two subsequent time steps (i.e. $u^{k+1}-u^{k}=0$ for some time step $k$). In other words, we indeed get the final steady segmentation for our threshold dynamic algorithm. From the experience of numerical simulations, it only needs $5-20$ steps for the images by choosing proper $\delta t$ and ${\lambda}$.

%In our tests, the data term has the form $g_i=|f-C_i|^{\beta}$, where $f$ is the input image and $C_i$ is the average intensity of the region $i$ which is determined by the previous step value.
%The timing results were
%obtained on a laptop with a 3.0GHz Intel(R) Core(TM) i7 processor
%and 8GB of RAM.\\
%

\begin{comment}
\begin{remark}
Before applying Algorithm 1, image data is normalized to be of zero mean and unit variance,
\begin{align*}
f_N(x) = \frac{1}{\sqrt{Var(f)}} [f(x) - Mean(f)].
\end{align*}
This step is important to make parameter choices largely independent of specific a image.
\end{remark}
\end{comment}

\subsection{Example 1: Cameraman}
We first test our algorithm on the standard cameraman image using two-phase segmentation. Figure~\ref{Fig:CameramanGiven} is the original  image.   We start with the initial contour  given in Fig.~\ref{Fig:CameramanInitial}.   We choose $\delta t = 0.03$ and  ${\lambda} = 0.01$.  Our algorithm takes only 15 iterations to converge to a complete steady state, i.e. $e^k=0$ (for $k=15$) with a total computation time of only $0.1188$ seconds.  Fig.~\ref{Fig:CameramanFinal} gives the final segmentation contour.   We also plot the normalized energy $ \mathcal{E}^{\delta t}/|\Omega|$  as a function of the iteration number $k$ in Fig.\ref{Fig:Cameramanenergy}, which verifies the monotone decay  of the energy.  In fact,  the energy decays quickly in the first few iterations and almost reaches steady state in less than 10 iterations.
\begin{figure}[h]
\centering
  \subfigure[Given Image.\label{Fig:CameramanGiven}]  {\includegraphics[width=4cm]{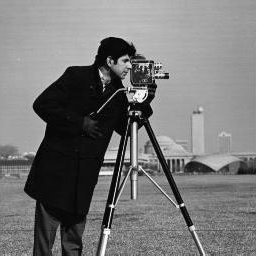} }
  \subfigure[Initial Contour.\label{Fig:CameramanInitial}]{\includegraphics[width=4cm]{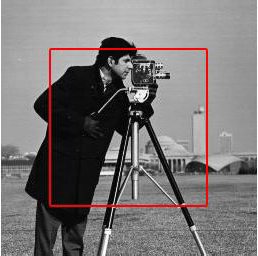}}
   \subfigure[Final Contour.\label{Fig:CameramanFinal}]{ \includegraphics[width=4cm]{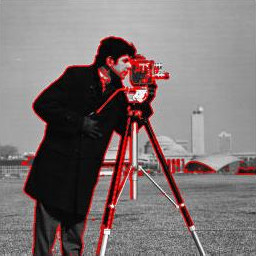}}
   \caption{Segmentation results for the classic cameraman image with $\delta t = 0.03$ and ${\lambda} = 0.01$. The algorithm  converges in 15 iterations with a computational time of $0.1188$ seconds}
\end{figure}

%  \begin{figure}[h]
   %  \begin{subfigure}[h]{0.4 \textwidth}
	%	\centering
	%	\includegraphics{image_segmentation/graph/energy_2.jpg}
	%	\caption{${\lambda} = 0.01$.}\label{Fig:CameramanELmd2}
%\end{figure}

\begin{figure}[h]
\centering
        \includegraphics[width=0.7\textwidth, height=0.4\textwidth]{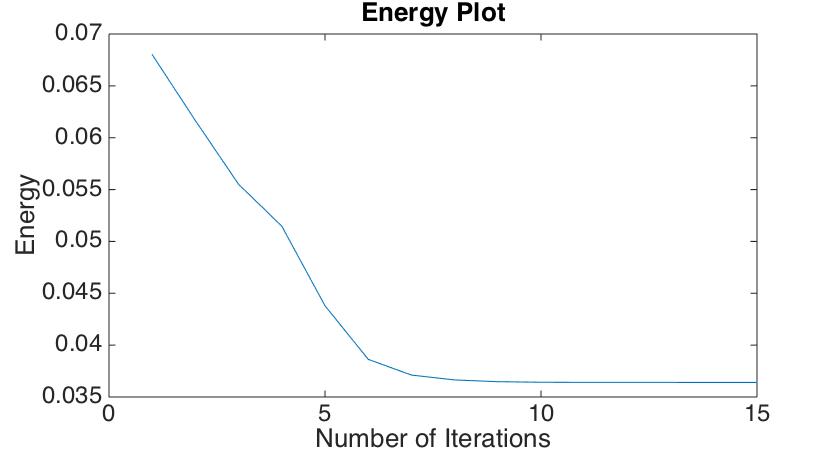}
  \caption{\label{Fig:Cameramanenergy}Energy curve for the iteration algorithm with $\delta t = 0.03$ and ${\lambda} = 0.01$. }
\end{figure}

To study the effect of the parameter $\lambda$ in  the energy (\ref{energy1}), we run our algorithm on the same test image for three different values of $\lambda=0.001, 0.01$ and $0.025$  but  with a fixed $\delta t=0.03$.   The final segmentation contours together with the energy curves are shown in Fig. \ref{Fig:CameramanELmd}. As the figure shows, larger $\lambda=0.025$ turns to smooth out the small-scale structures while smaller $\lambda=0.001$ would pick up more noisy regions.  This is easy to understand since $\lambda$ measures the relative importance of  the contour length  and  the data term in the Chan-Vese  functional to be minimized.  A larger $\lambda$ tends to shorten the total contour length and therefore does not favor  small-scale structures.  On the other hand,  convergence is much faster for a smaller $\lambda$ while a  larger $\lambda$ would require more iterations to converge as shown by the energy curves.

%Our relaxed model approaches to the exact Chan-Vese functional as $\delta t$ converges to $0$. At that limit, ${\lambda}$ plays the role of regularizing contour smoothness. This is also true for finite $\delta t$, as shown in Figure~\ref{Fig:CameramanELmd}.

\begin{figure}[h]
\centering
\subfigure{\includegraphics[width=4cm]{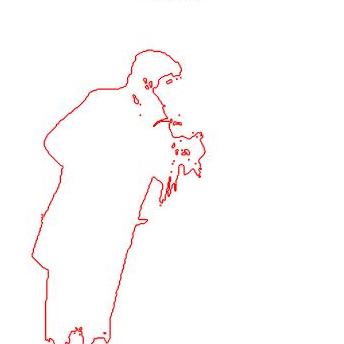} }
\subfigure{\includegraphics[width=4cm]{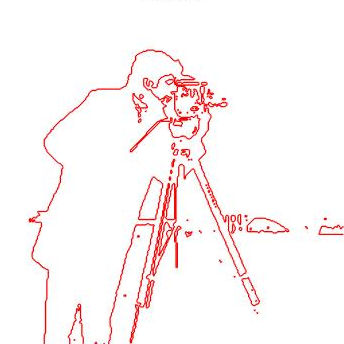}}
\subfigure{\includegraphics[width=4cm]{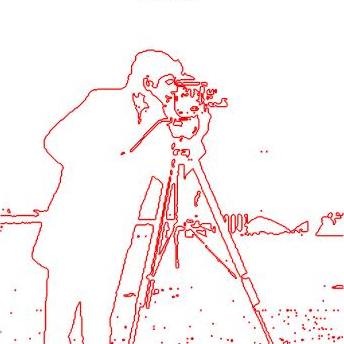}}		
  ~
\subfigure[${\lambda} = 0.025$. \label{Fig:CameramanELmd1}]{\includegraphics[width=4cm]{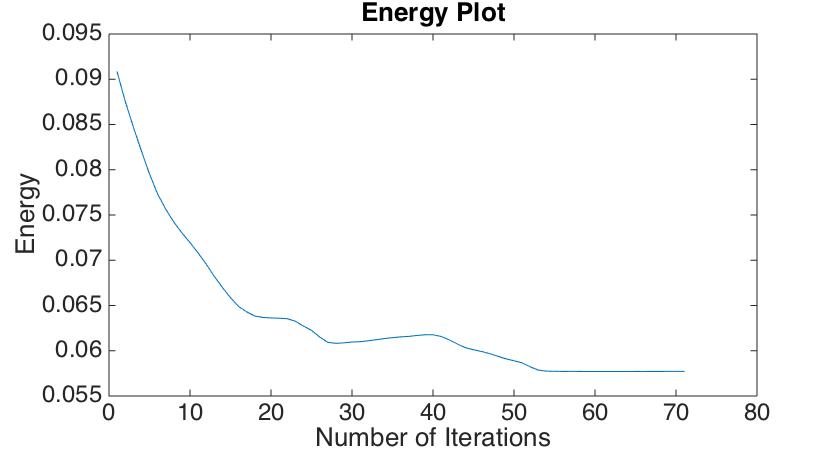}}		
\subfigure[${\lambda} = 0.01$.\label{Fig:CameramanELmd2}]{\includegraphics[width=4cm]{energy_2.jpg}}
\subfigure[${\lambda} = 0.001$.\label{Fig:CameramanELmd3}]{ \includegraphics[width=4cm]{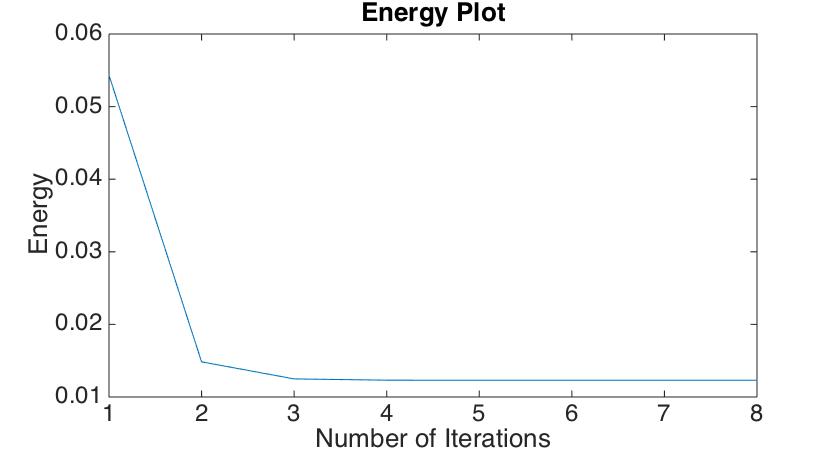}}	
\caption{Segmentation contours  and energy curves for  $\delta t= 0.03$ and different ${\lambda} $ values. }
	\label{Fig:CameramanELmd}
\end{figure}

\subsection{Example 2: A synthetic four-phase image}  We next use a synthetic color image  given in Fig. \ref{Fig:10242}. The image $f$ is a vector-valued function.   Gaussian noise is added with mean $0$ and variance $0.04$ to each component of image $f$.  The initial contours are given in Fig.~\ref{Fig:512i}.  We apply our four-phase  algorithm   to the image with three different resolutions from $128\times128$
to $512\times 512$. In each case,   $\delta t=0.01$ and ${\lambda}=0.003$. The algorithm converges in $7\sim8$ iterations  for all resolutions with runtimes of $0.0444, 0.1333, 0.6706$ seconds respectively, which demonstrates good stability of and robustness of our method.   Figures. \ref{Fig:128c}-\ref{Fig:512c} show the final segmentation result.

\begin{figure}[h]
\centering
  \subfigure[Image with Noise.\label{Fig:10242}]{\includegraphics[width=4cm]{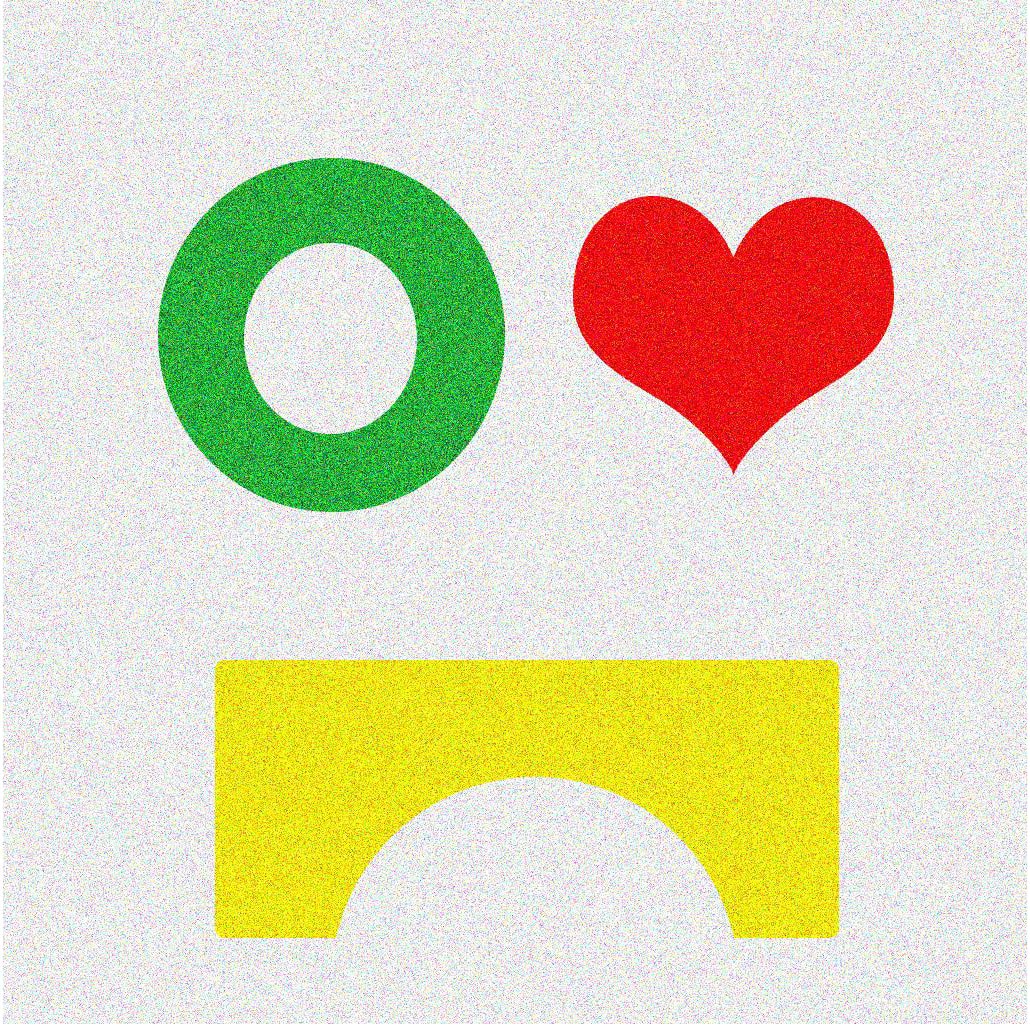}}
  \subfigure[Initial Contour.\label{Fig:512i}]{\includegraphics[width=4cm]{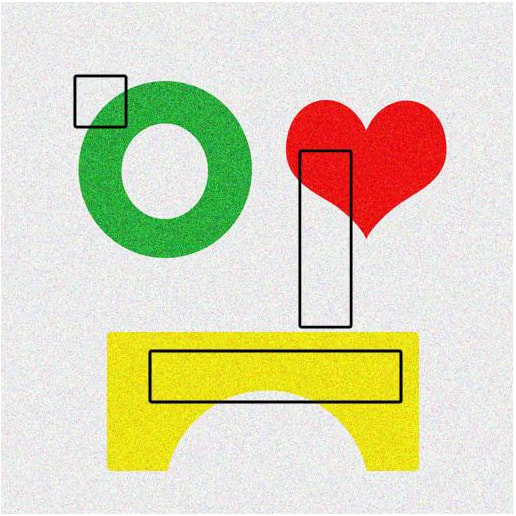}}
  \\
  \subfigure[$128\times 128$.\label{Fig:128c}]{\includegraphics[width=4cm]{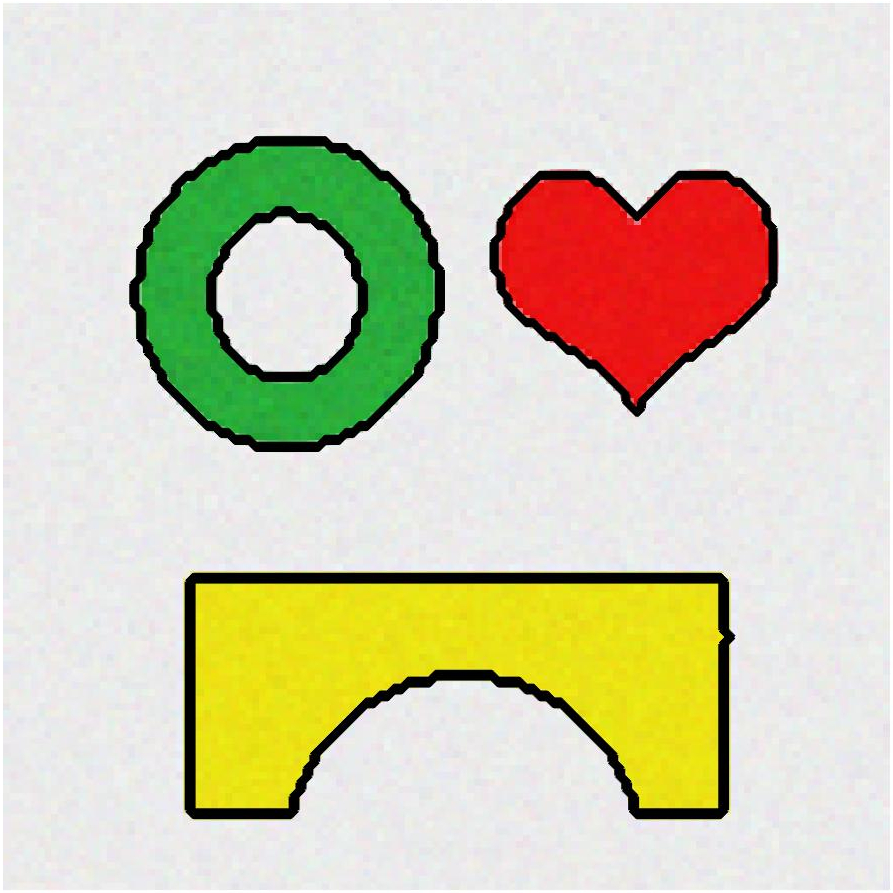}}
   \subfigure[$256\times 256$.\label{Fig:256c}]{ \includegraphics[width=4cm]{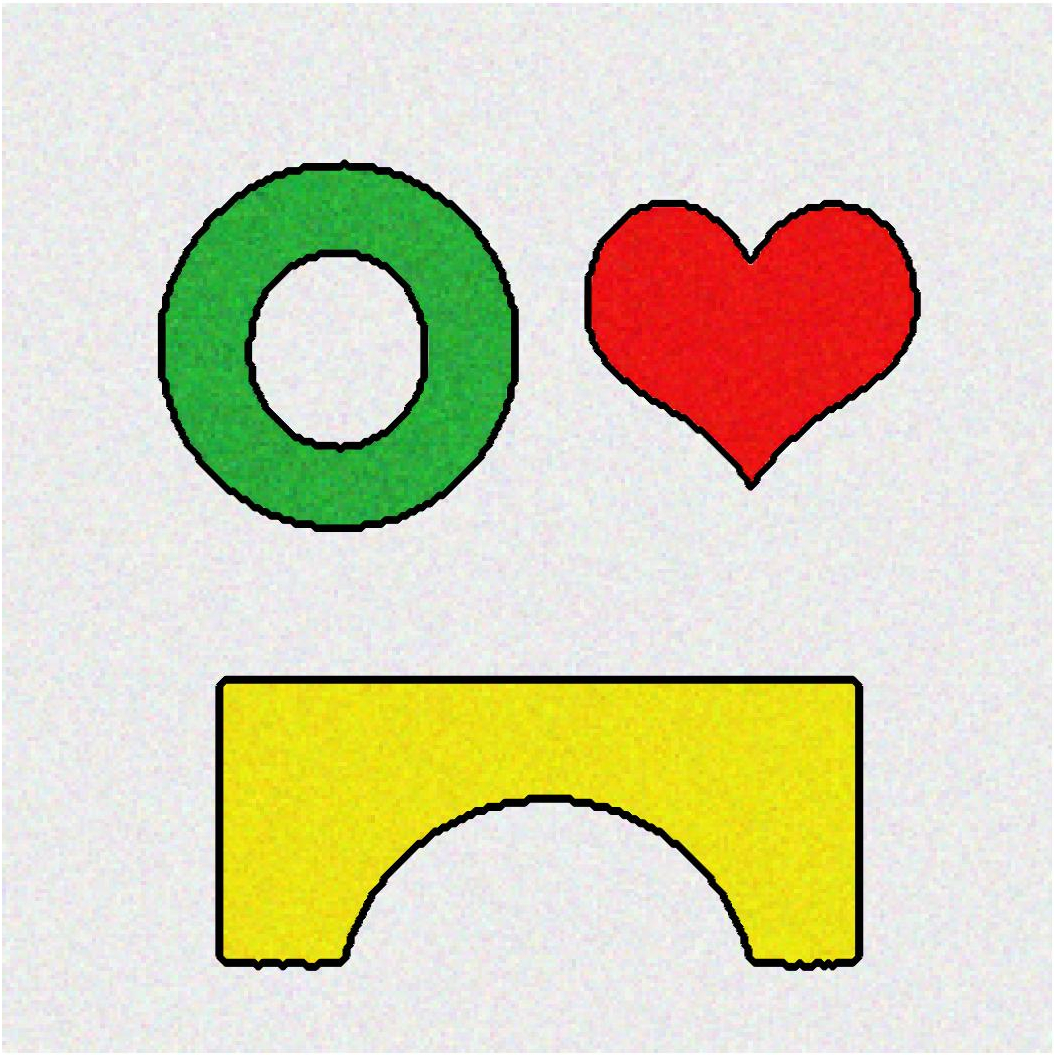}}
  \subfigure[$512\times 512$.\label{Fig:512c}]{\includegraphics[width=4cm]{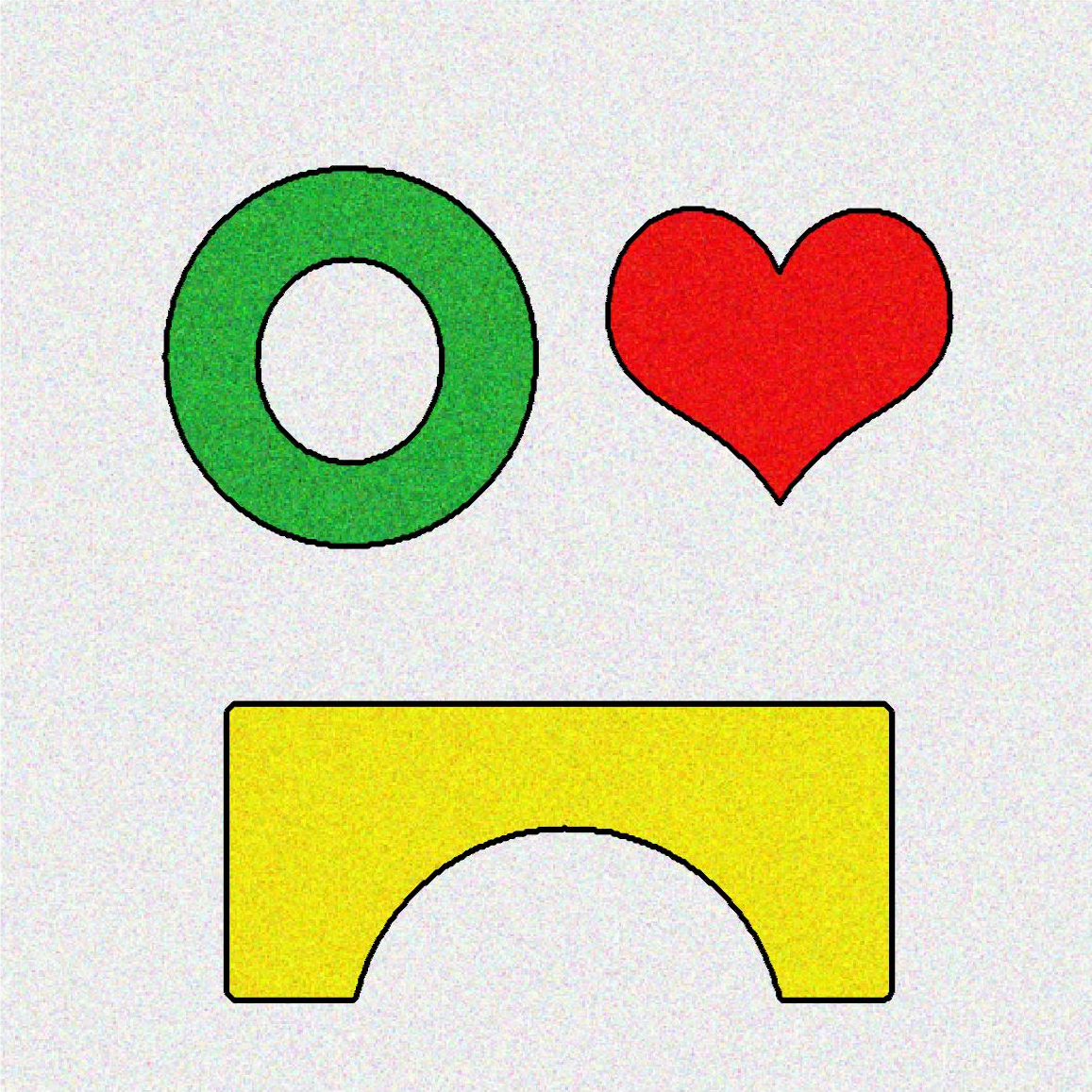}}
   \caption{Segmentation for images with different resolutions and with the parameters   $\delta t=0.01$ and ${\lambda}=0.003$} \label{Fig:MultiResolution}
\end{figure}

\subsection{Example 3: Flower color image} We now consider an image containing flowers of different colors  in
Fig.~\ref{Fig:Flowers}.  We first use a two-phase segmentation algorithm with $\delta t = 0.01$ and ${\lambda} = 0.005$ and the initial contour in Fig. \ref{Fig:FlowersInitial2}. The algorithm converges in $20$ iterations with a runtime of $0.6751$ seconds. The final segmentation result is given in Fig.~\ref{Fig:FlowersFinal2}.  We also use a four-phase segmentation algorithm with $\delta t = 0.01$ and ${\lambda} = 0.003$ and the initial contour in Fig. \ref{Fig:FlowersInitial4}. The algorithm converges in $18$ iterations with a runtime of $1.1007$ seconds.  The final segmentation result is given in Fig.~\ref{Fig:FlowersFinal4} and \ref{Fig:FlowersFinal}

\begin{figure}[h]
\centering
 \subfigure[Given Color Image. \label{Fig:Flowers}]{\includegraphics[width=4cm]{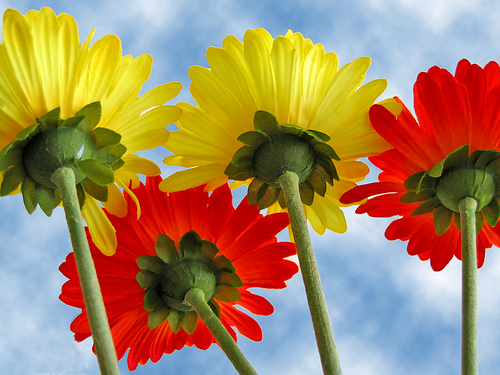}}
~
\subfigure[Initial Contour. \label{Fig:FlowersInitial2}]{\includegraphics[width=4cm]{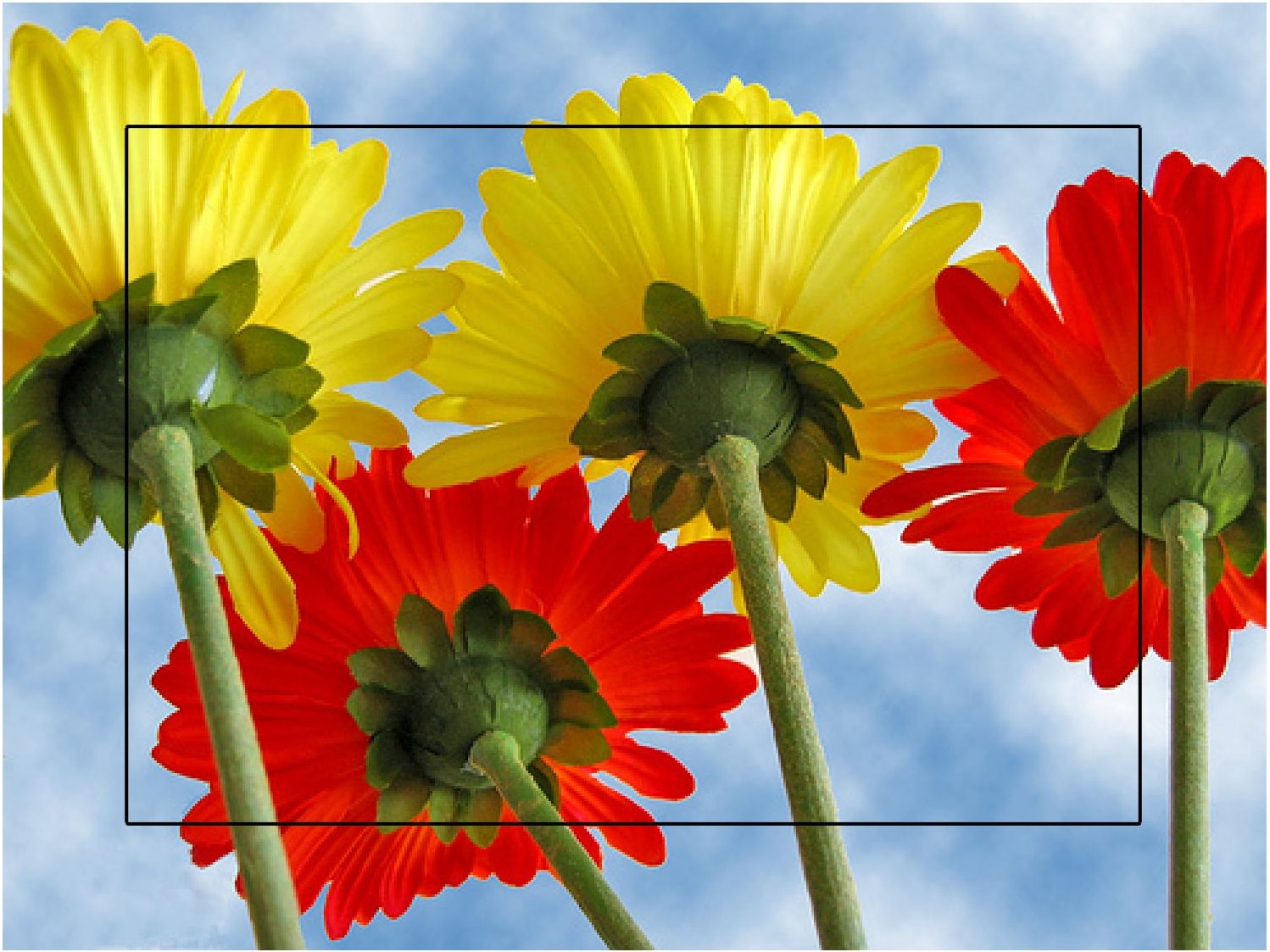}}
~
\subfigure[Final Contour.\label{Fig:FlowersFinal2}]{ \includegraphics[width=4cm]{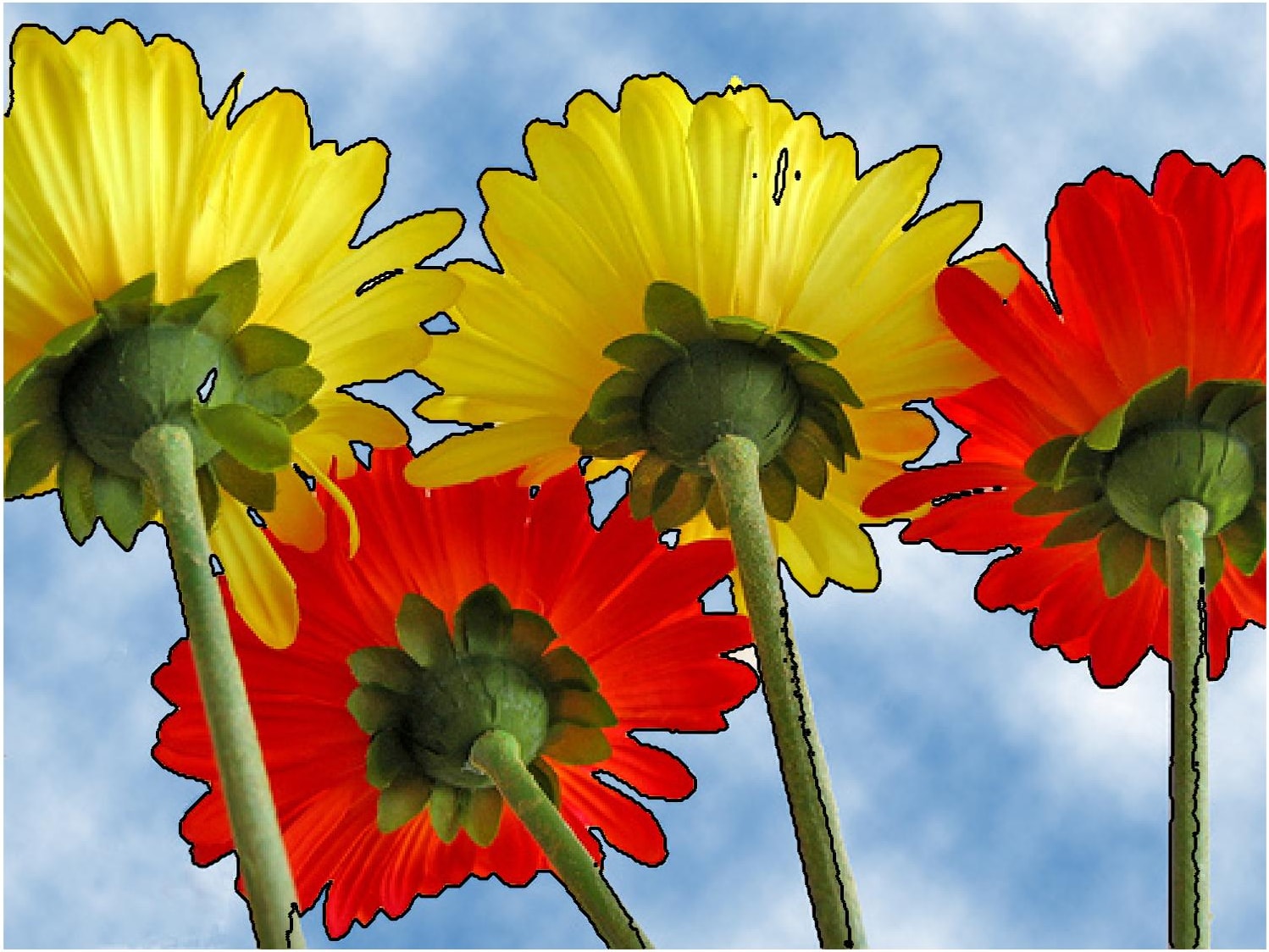}}
 \caption{Two-phase segmentation for a $375\times 500$ RGB image and with parameters   $\delta t = 0.01$ and  ${\lambda} = 0.005$.}
 \end{figure}

   \begin{figure}
   \centering
  \subfigure[Initial Contour.\label{Fig:FlowersInitial4}]{  \includegraphics[width=4cm]{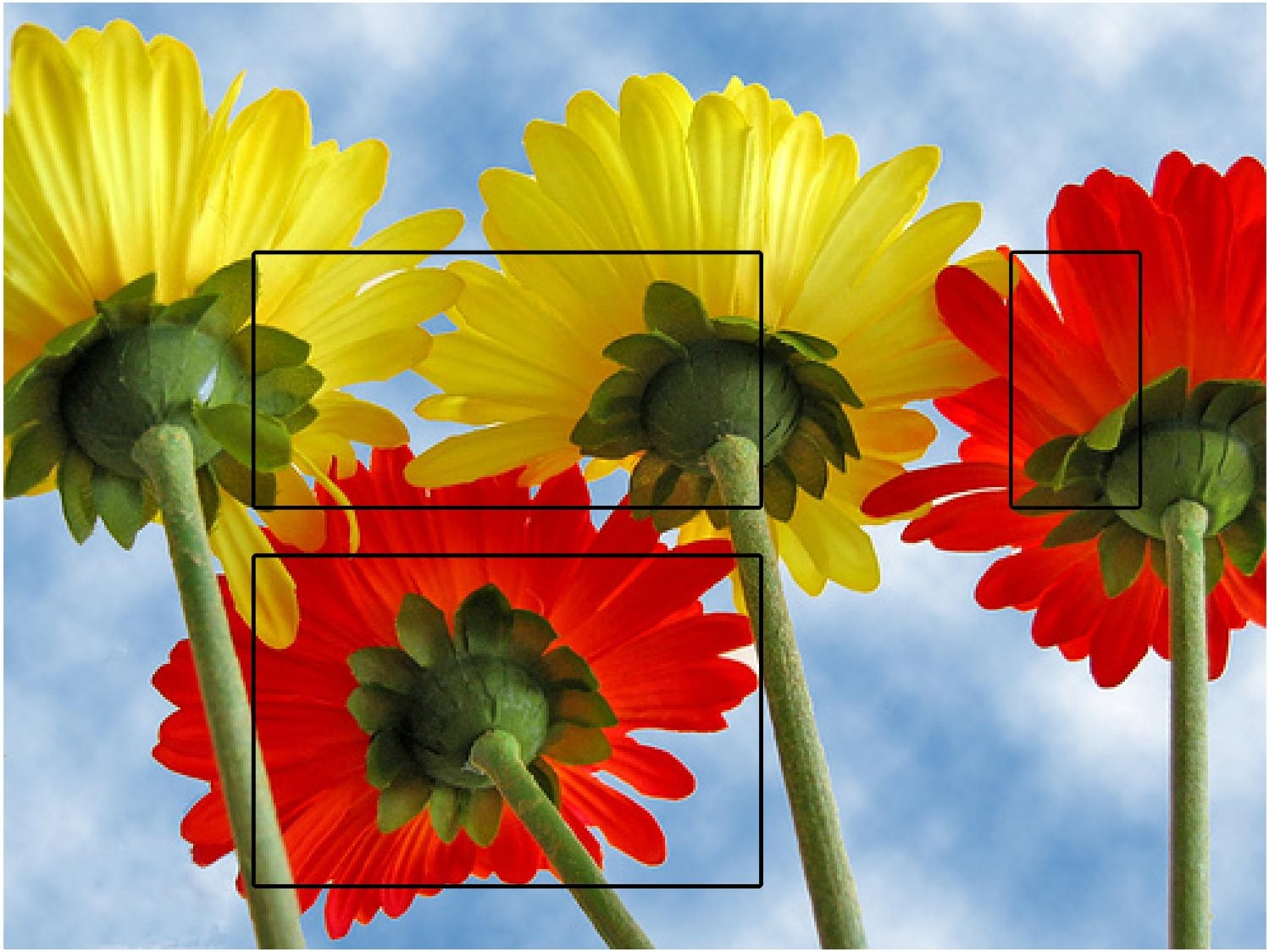}}
 ~
  \subfigure[Final Contour.\label{Fig:FlowersFinal4}]{ \includegraphics[width=4cm]{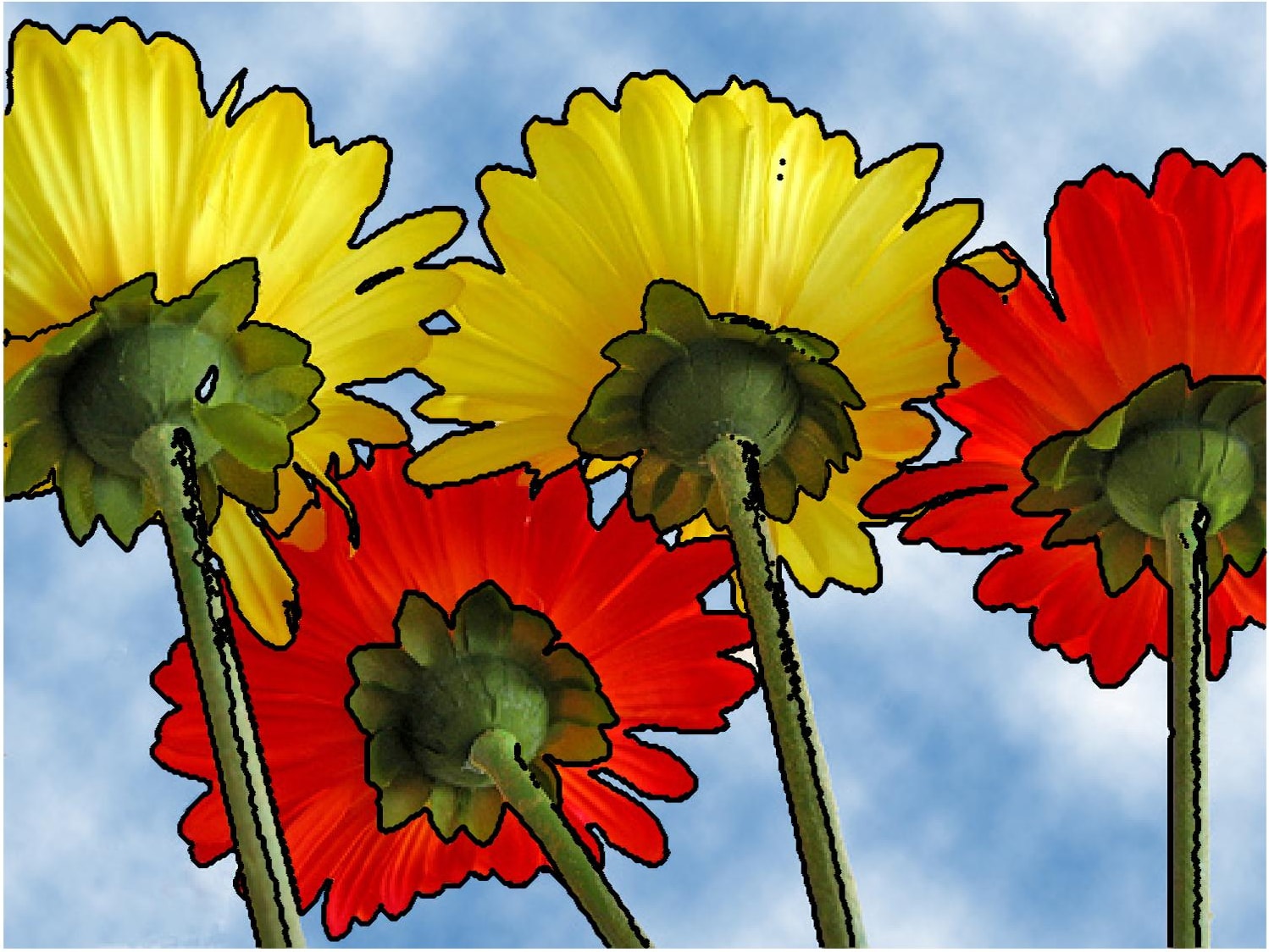}}
    ~
 \subfigure[Four Segments.\label{Fig:FlowersFinal}]{  \includegraphics[width=4cm]{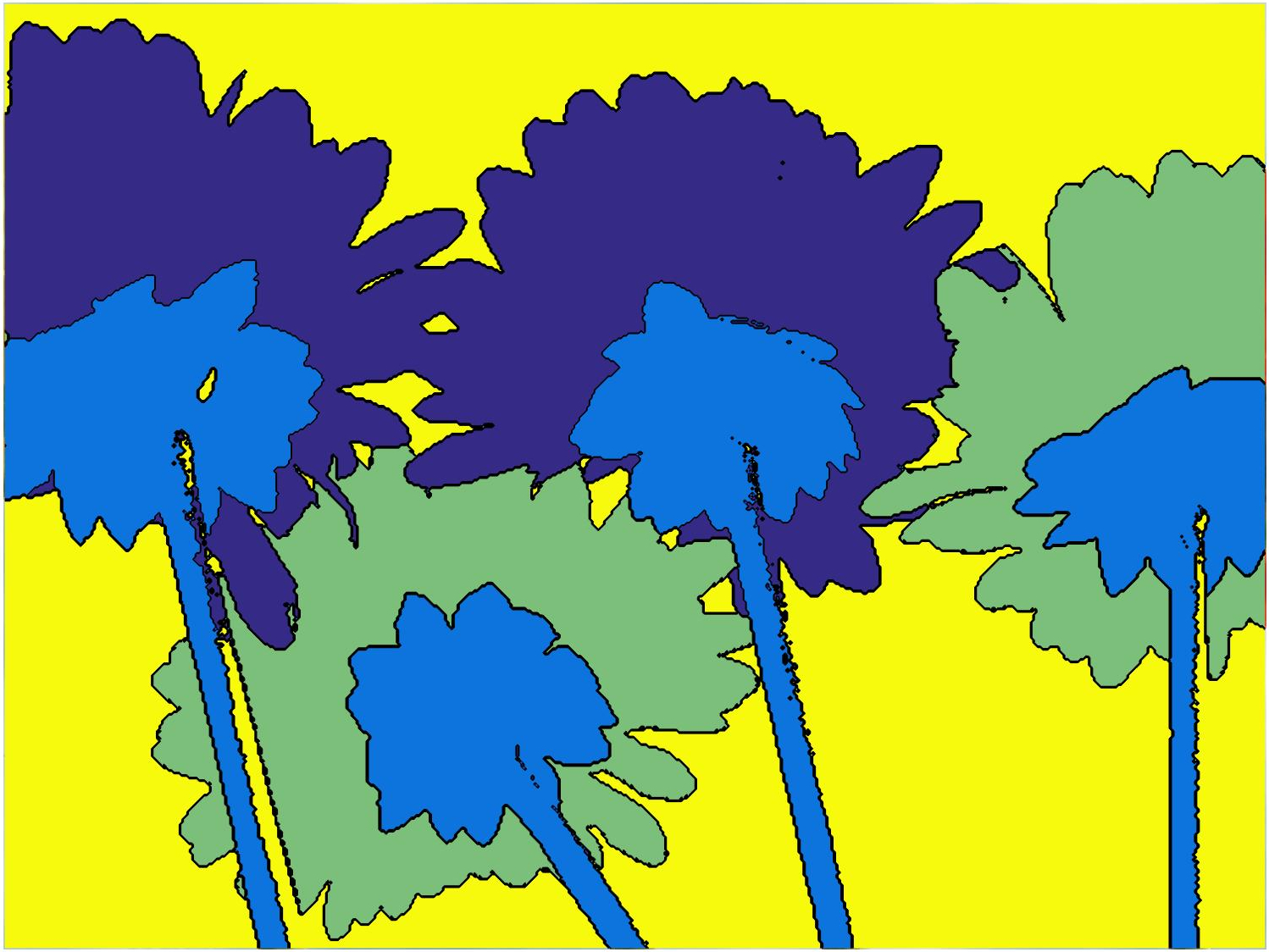}}
   \caption{Four phase segmentation for a $375\times 500$ RGB image with  $\delta t = 0.01$ and ${\lambda} = 0.003$.}
\end{figure}

\section{Conclusions}
We have proposed an efficient iterative thresholding algorithm for the Chan-Vese model for multi-phase  image segmentation.
The algorithm works by alternating the convolution step with the thresholding step and has the optimal computational complexity of $O(N \log N)$ per iteration.  We  prove that the iterative algorithm has the property of  total energy decay.  The numerical results show that the method is stable and the number of iterations before convergence is independent of the spacial resolution (for a given image).  The relative importance of the different effects in the energy functional is studied by tuning the parameter $\lambda$.  Our numerical results also show that the proposed method is competitive (in terms of efficiency) with  many existing methods for image segmentation.

\bigskip

\appendix
{
\section{Proof of Lemma \ref{Lemma}}

We  prove the lemma for the general case that $n\ge 2$ and $d\ge 1$ (i.e. $f$ is a $d$-dimensional vector valued function) by contradiction. Let $v=(v_1,\cdots,v_n)\in \mathcal{K}$ be a minimizer of  $\mathcal{E}^{\delta t}(u)+\mathcal{L}(u)$ on $\mathcal{K}$. If  $v \notin \mathcal{S}$, then  there exists a set $A \subseteq \Omega$ ($|A|>0$) and a constant $0 < \epsilon < \frac{1}{2}$ such that  for some $k,l \in \{1,\cdots,n\}$ with $k \neq l$,
\begin{align*}
v_k(x),v_l(x)\in (\epsilon,1-\epsilon), \quad  \forall x\in A.
\end{align*}
Denote
\begin{align*}
u_m^t(x,t)=v_m(x)+t(\delta_{m,l}-\delta_{m,k}) \chi_{A}(x)
\end{align*}
for $m=1,\cdots,n$ where $\chi_{A}(x)$ represents the characteristic function of region $A$ and
\begin{align*}
\delta_{m,l}=\left\lbrace\begin{array}{cc}
1 & m=l\\
0 & m \neq l.
\end{array}\right.
\end{align*}
When $-\epsilon\leq t \leq \epsilon$, we have $u_m^t(x,t)\geq 0$ and $\sum\limits_{m=1}^{n}u_m^t(x,t)=1$ so that $u^t(x,t)=(u_1^t(x,t),\cdots,u_n^t(x,t))\in \mathcal{K}$.
Now denote
\begin{align}
f^m = \int_{\Omega} v_m fd \Omega, \quad  V^m =\int_{\Omega} v_m d \Omega,  \quad  f^A = \int_{\Omega} \chi_{A} fd \Omega. \label{Eq:Notation1}
\end{align}
Then
\begin{align}
\int_{\Omega} u_m^t f d\Omega &=\int_{\Omega} v_m f d\Omega+t\int_{\Omega} (\delta_{ml}-\delta_{mk}) \chi_{A} f d\Omega \nonumber \\
&= f^m+t (\delta_{ml}-\delta_{mk})f^A\\
\int_{\Omega} u_m^t  d\Omega &=\int_{\Omega} v_m  d\Omega+t\int_{\Omega} (\delta_{ml}-\delta_{mk}) \chi_{A} d\Omega \nonumber \\
&= V^m+t (\delta_{ml}-\delta_{mk})|A|
\end{align}
Let
\[  C_m=\frac{\int_{\Omega} u_m^t f d\Omega}{\int_{\Omega} u_m^t  d\Omega}. \]
It is easy to see that $C_m$ depends on $t$ only
when $m=l$ or $k$. We have
$$C_l=\frac{f^l+tf^A}{V^l+t|A|} \quad  \text{and} \quad  C_k=\frac{f^k-tf^A}{V^k-t|A|}.$$
Then, we calculate the first  and second order derivatives of $C_l$ and $ C_k$ with respect to $t$ as follows:
\begin{align}
\frac{d C_l}{d t}&=\frac{f^A}{V^l+t|A|}-\frac{|A|(f^l+tf^A)}{(V^l+t|A|)^2} \nonumber \\
\frac{d C_k}{d t}&=-\frac{f^A}{V^k-t|A|}+\frac{|A|(f^k-tf^A)}{(V^k-t|A|)^2} \nonumber\\
\frac{d^2 C_l}{d t^2}&=-\frac{2|A|f^A}{(V^l+t|A|)^2}+\frac{2|A|^2(f^l+tf^A)}{(V^l+t|A|)^3} \label{Eq:RelaxDeri} \\
\frac{d^2 C_k}{d t^2}&=-\frac{2|A|f^A}{(V^k-t|A|)^2}+\frac{2|A|^2(f^k-tf^A)}{(V^k-t|A|)^3} \nonumber
\end{align}
A direct calculation then gives
\begin{align}
\frac{d^2 \mathcal{E}^{\delta t}}{d t^2} =& \int_{\Omega} \sum\limits_{i=1}^{n} \left(4\frac{d u_i^t}{d t }\langle C_i-f,\frac{d C_i}{d t}\rangle+2 u_i^t \langle C_i-f,\frac{d ^2 C_i}{d t^2}\rangle+2 u_i^t ||\frac{d  C_i}{d t}||_2^2\right) d\Omega.\nonumber \\
&-4\frac{\lambda\sqrt{\pi}}{\sqrt{\delta t}}\int_{\Omega}\chi_{A}G_{\delta t}*\chi_{A}d\Omega \nonumber\\
=&4 \int_{\Omega} \chi_{A} \langle C_l-f, \frac{d C_l}{d t} \rangle d\Omega -4 \int_{\Omega} \chi_{A} \langle C_k-f,\frac{d C_k}{d t} \rangle d\Omega \label{Eq:RelaxDoubleD} \\
&+2 \int_{\Omega}u_l^t\langle C_l-f,\frac{d ^2 C_l}{d t^2}\rangle d\Omega +2 \int_{\Omega}u_k^t \langle C_k-f ,\frac{d ^2 C_k}{d t^2}\rangle d\Omega \nonumber \\
&+2 \int_{\Omega}u_l^t ||\frac{d  C_l}{d t}||_2^2d\Omega +2 \int_{\Omega}u_k^t||\frac{d C_k}{d t}||_2^2d\Omega \nonumber \\
&-4\frac{\lambda\sqrt{\pi}}{\sqrt{\delta t}}\int_{\Omega}\chi_{A}G_{\delta t}*\chi_{A}d\Omega  \nonumber
\end{align}
where $\langle\alpha,\beta\rangle = \sum\limits_{i=1}^n \alpha_i\beta_i $ for $\alpha, \beta \in R^n$.
Evaluating at $t=0$ and substituting (\ref{Eq:RelaxDeri}) into (\ref{Eq:RelaxDoubleD}), we have
\begin{align}
\left. \frac{d^2 \mathcal{E}^{\delta t}}{d t^2}\right|_{t=0}=&  \int_{\Omega}4 \chi_{A}\langle \frac{f^l}{V^l}-f,\frac{f^A}{V^l}-\frac{|A|f^l}{(V^l)^2} \rangle +4  \chi_{A}\langle \frac{f^k}{V^k}-f, \frac{f^A}{V^k}-\frac{|A|f^k}{(V^k)^2}\rangle d\Omega \label{Eq:Relax1stTerm}\\
&+2 \int_{\Omega}v_l\langle \frac{f^l}{V^l}-f, -\frac{2|A|f^A}{(V^l)^2}+\frac{2|A|^2f^l}{(V^l)^3} \rangle d\Omega  \label{Eq:Relax2ndTerm} \\
&+2 \int_{\Omega}v_k\langle  \frac{f^k}{V^k}-f, -\frac{2|A|f^A}{(V^k)^2}+\frac{2|A|^2f^k}{(V^k)^3}\rangle d\Omega  \label{Eq:Relax3rdTerm} \\
&+2 \int_{\Omega}v_l ||\frac{f^A}{V^l}-\frac{|A|f^l}{(V^l)^2}||_2^2 d\Omega\label{Eq:Relax4thTerm}\\
&+2 \int_{\Omega}v_k ||\frac{f^A}{V^k}-\frac{|A|f^k}{(V^k)^2}||_2^2 d\Omega \label{Eq:Relax5thTerm}\\
& -4\frac{\lambda\sqrt{\pi}}{\sqrt{\delta t}}\int_{\Omega}\chi_{A}G_{\delta t}*\chi_{A}d\Omega.\label{Eq:Relax6thTerm}
\end{align}
Then, using (\ref{Eq:Notation1}) and the definition of $|A|$, we can calculate the above integrals (note  that $f^l,f^k, f^A,V^l,V^k$ and $|A|$ in the integrand are all independent of $\Omega$).  Therefore,
\begin{align}
&(\ref{Eq:Relax1stTerm})+(\ref{Eq:Relax4thTerm}) +(\ref{Eq:Relax5thTerm})\nonumber\\
=&-\frac{2}{V^l}||\frac{|A|f^l}{V^l}-f^A||_2^2-\frac{2}{V^k}||\frac{|A|f^k}{V^k}-f^A||_2^2 <0.
\end{align}
Similarly, direct calculations show that $(\ref{Eq:Relax2ndTerm})=0$ and $(\ref{Eq:Relax3rdTerm})=0$.
It is obvious that
\begin{align*}
-4\frac{\lambda\sqrt{\pi}}{\sqrt{\delta t}}\int_{\Omega}\chi_{A}G_{\delta t}*\chi_{A}d\Omega<0.
\end{align*}
Combining the above, we have
\begin{align*}
\left. \frac{d^2 \mathcal{E}^{\delta t}}{d t^2}\right|_{t=0}<0.
\end{align*}
Thus, $v(x) = u(x, 0)$ cannot be a minimizer. This contradicts the assumption.

\section{Proof of Theorem \ref{Theorem}}
%By the definition of the linearization and the minimization problem (\ref{Eq:MinLinearEnergy}) we %solved, we know
From (\ref{linear1}), we have
\begin{align*}
&\mathcal{E}^{\delta t}(u_1^k,\cdots,u_n^k)+\sum\limits_{i=1}^n\int_{\Omega} \sum\limits_{j\neq i, j=1}^n\frac{\lambda\sqrt{\pi}}{\sqrt{\delta t}}u_i^k G_{\delta t}*u_j^k d\Omega = \mathcal{L}(u_1^k,\cdots,u_n^k,u_1^k,\cdots,u_n^k)\\
&\geq \mathcal{L}(u_1^{k+1},\cdots,u_n^{k+1},u_1^k,\cdots,u_n^k) =\mathcal{E}^{\delta t}(u_1^{k+1},\cdots,u_n^{k+1})\\
&+\sum\limits_{i=1}^{n}\int_{\Omega}\left(u_i^{k+1}(g_i^k-g_i^{k+1})+ \sum\limits_{ j=1,j\neq i}^{n} \frac{2\lambda\sqrt{\pi}}{\sqrt{\delta t}} u_i^{k+1} G_{\delta t} *u_j^k\right)d\Omega \\
&-\sum\limits_{i=1}^n\int_{\Omega} \sum\limits_{j\neq i, j=1}^n\frac{\lambda\sqrt{\pi}}{\sqrt{\delta t}}u_i^{k+1} G_{\delta t}*u_j^{k+1} d\Omega.
\end{align*}
That leads to
\begin{align}
\mathcal{E}^{\delta t}(u_1^k,\cdots,u_n^k)\geq\mathcal{E}^{\delta t}(u_1^{k+1},\cdots,u_n^{k+1})+I  \label{Eq:StabilityIneq}
\end{align}
with
\begin{align*}
I=&\sum\limits_{i=1}^{n}\int_{\Omega}\left(u_i^{k+1}(g_i^k-g_i^{k+1})+ \sum\limits_{ j=1,j\neq i}^{n} \frac{2\lambda\sqrt{\pi}}{\sqrt{\delta t}} u_i^{k+1} G_{\delta t} *u_j^k\right)d\Omega \\
&-\sum\limits_{i=1}^n\int_{\Omega} \sum\limits_{j\neq i, j=1}^n\frac{\lambda\sqrt{\pi}}{\sqrt{\delta t}}u_i^{k+1} G_{\delta t}*u_j^{k+1} d\Omega\\
&-\sum\limits_{i=1}^n\int_{\Omega} \sum\limits_{j\neq i, j=1}^n\frac{\lambda\sqrt{\pi}}{\sqrt{\delta t}}u_i^k G_{\delta t}*u_j^k d\Omega \\
=&I_1+I_2
\end{align*}
where \begin{align*}
I_1=&\sum\limits_{i=1}^{n}\int_{\Omega}u_i^{k+1}(g_i^k-g_i^{k+1}) d\Omega\\
I_2=&\sum\limits_{i=1}^{n}\sum\limits_{ j=1,j\neq i}^{n}\int_{\Omega} \frac{\lambda\sqrt{\pi}}{\sqrt{\delta t}} u_i^{k+1} G_{\delta t} *(u_j^k-u_j^{k+1}) d\Omega\\
&-\sum\limits_{i=1}^{n}\sum\limits_{ j=1,j\neq i}^{n}\int_{\Omega} \frac{\lambda\sqrt{\pi}}{\sqrt{\delta t}} (u_i^{k}-u_i^{k+1}) G_{\delta t} *u_j^k d\Omega.
\end{align*}
Now, we only need to prove that $I_1\geq0$ and $I_2\geq0$. From the definition of $C_i^{k+1}$ and using  the fact that $\int_{\Omega}u_i^{k+1}f d\Omega =\int_{\Omega}u_i^{k+1}d\Omega C_i^{k+1}$, we have
\begin{align}
I_1=&\sum\limits_{i=1}^{n}\int_{\Omega}u_i^{k+1}(||C_i^k-f||_2^2-||C_i^{k+1}-f||_2^2) d\Omega \nonumber\\
=&\sum\limits_{i=1}^{n}\int_{\Omega}u_i^{k+1} (||C_i^k||_2^2-||C_i^{k+1}||_2^2-2\langle C_i^k-C_i^{k+1}, f \rangle) d\Omega \nonumber\\
=&\sum\limits_{i=1}^{n}\left\lbrace\int_{\Omega}u_i^{k+1}d\Omega (||C_i^k||_2^2-||C_i^{k+1}||_2^2-2\langle C_i^k-C_i^{k+1}, C_i^{k+1} \rangle) \right\rbrace \label{Eq:StabilityI1}\\
=&\sum\limits_{i=1}^{n}\left\lbrace\int_{\Omega}u_i^{k+1}d\Omega ||C_i^k-C_i^{k+1}||_2^2 \right\rbrace \geq 0.\nonumber
\end{align}
By changing the order of the two summations in the second part of $I_2$ and using the fact that $\sum\limits_{i=1}^n u_i^k=1$ for any $k$, we obtain
\begin{align}
I_2=&\sum\limits_{i=1}^{n}\sum\limits_{ j=1,j\neq i}^{n}\int_{\Omega} \frac{\lambda\sqrt{\pi}}{\sqrt{\delta t}} (u_i^{k+1}-u_i^k) G_{\delta t} *(u_j^k-u_j^{k+1}) d\Omega \nonumber\\
 =&\sum\limits_{i=1}^{n}\int_{\Omega} \frac{\lambda\sqrt{\pi}}{\sqrt{\delta t}} (u_i^{k+1}-u_i^k) G_{\delta t} *\left(\sum\limits_{ j=1,j\neq i}^{n}(u_j^k-u_j^{k+1}) \right)d\Omega \nonumber \\
=&\sum\limits_{i=1}^{n}\int_{\Omega} \frac{\lambda\sqrt{\pi}}{\sqrt{\delta t}} (u_i^{k+1}-u_i^k) G_{\delta t} *(1-u_i^k-(1-u_i^{k+1}))d\Omega \label{Eq:StabilityI2}\\
=&\sum\limits_{i=1}^{n}\int_{\Omega} \frac{\lambda\sqrt{\pi}}{\sqrt{\delta t}} (u_i^{k+1}-u_i^k) G_{\delta t} *(u_i^{k+1}-u_i^k)d\Omega\geq0. \nonumber
\end{align}
Combining (\ref{Eq:StabilityIneq}),  (\ref{Eq:StabilityI1}) and (\ref{Eq:StabilityI2})  gives (\ref{Eq:StabilityInProof}).
}

\bibliographystyle{siam}
\bibliography{reference1}
\end{document}